\documentclass[9.5bp,journal,compsoc]{IEEEtran}
\usepackage{epsfig}
\usepackage{graphicx}
\usepackage[pagebackref=false,breaklinks=true,colorlinks,linkcolor={black},citecolor={black},urlcolor={blue},bookmarks=false]{hyperref}
\usepackage{cite}
\usepackage{microtype}
\usepackage{subfigure}
\usepackage{url}

\usepackage{array}
\usepackage[normalem]{ulem}
\usepackage{ulem}
\usepackage{amsmath}
\usepackage{graphicx,subfigure}
\usepackage{longtable}
\usepackage{rotating}
\usepackage{multirow}
\usepackage{makecell}
\usepackage{bm}
\usepackage{float}
\usepackage{flushend}
\usepackage{xcolor}
\usepackage{soul}
\usepackage{booktabs}
\usepackage{threeparttable}
\usepackage{color}
\usepackage{algorithm}
\usepackage{algorithmic}
\usepackage{tabularx}
\usepackage{amsfonts}
\usepackage{epstopdf}
\usepackage{latexsym}
\usepackage{amssymb}
\usepackage{footnote}

\usepackage{pifont}

\newtheorem{theorem}{Theorem}
\newtheorem{lemma}{Lemma}
\newtheorem{proof}{Proof}
\newtheorem{definition}{Definition}
\newtheorem{remark}{Remark}
\renewcommand{\IEEEQED}{\IEEEQEDopen}

\usepackage{threeparttable}

\def\Sp{{\scriptsize{\textcircled{{\emph{\tiny{\textbf{Sp}}}}}}}}
\def\etal{\textit{et al.}}
\def\ie{\textit{i.e.}}

\begin{document}

\title{Label Learning Method Based on Tensor Projection}

\author{Jing Li,
        Quanxue~Gao,
        Qianqian~Wang,
        Cheng~Deng,
        and~Deyan~Xie
\IEEEcompsocitemizethanks{
\IEEEcompsocthanksitem This work was supported in part by xxxxx. (Corresponding author: Q. Gao.)\protect
\IEEEcompsocthanksitem J. Li, Q. Gao, Q. Wang are with the State Key laboratory of Integrated Services Networks, Xidian University, Xi'an 710071, China (e-mail: qxgao@xidian.edu.cn).\protect

%\IEEEcompsocthanksitem X. Gao is with the Chongqing Key Laboratory of Image Cognition, Chongqing University of Posts and Telecommunications, Chongqing 400065, China (e-mail: gaoxb@cqupt.edu.cn), and with the School of Electronic Engineering, Xidian University, Xi’an 710071, China (e-mail: xbgao@mail.xidian.edu.cn).
\protect}%
\thanks{Manuscript received XXXX; revised XXXX; accepted XXXX.}}

\markboth{IEEE TRANSACTIONS}%
%\markboth{Journal of \LaTeX,~Vol.~XXX,~No.~XXX,~XXX~XXX}%
{Shell \MakeLowercase{\textit{Xia et al.}}: Multi-View Clustering via Semi-non-negative Tensor Factorization}

\IEEEtitleabstractindextext{%
\begin{abstract}
Multi-view clustering method based on anchor graph has been widely concerned due to its high efficiency and effectiveness. In order to avoid post-processing, most of the existing anchor graph-based methods learn bipartite graphs with connected components. However, such methods have high requirements on parameters, and in some cases it may not be possible to obtain bipartite graphs with clear connected components. To end this, we propose a label learning method based on tensor projection (LLMTP). Specifically, we project anchor graph into the label space through an orthogonal projection matrix to obtain cluster labels directly. Considering that the spatial structure information of multi-view data may be ignored to a certain extent when projected in different views separately, we extend the matrix projection transformation to tensor projection, so that the spatial structure information between views can be fully utilized. In addition, we introduce the tensor Schatten $p$-norm regularization to make the clustering label matrices of different views as consistent as possible. Extensive experiments have proved the effectiveness of the proposed method.
\end{abstract}

}

\maketitle
\IEEEdisplaynontitleabstractindextext
\IEEEpeerreviewmaketitle

\IEEEraisesectionheading{\section{Introduction}\label{sec:introduction}}
\IEEEPARstart{A}{s} one of the common techniques of data mining, clustering can be used to discover the internal structure and organization of data and divide them into different meaningful clusters. With the wide application of various sensors and other technologies, the description of the same object is more and more prone to diversification and isomerization. For example, an event can be described by a text, a picture, a voice and a video; A picture can be described with different features, such as GIST, CMT, HOG, etc. A section of path information in automatic driving can be represented by liDAR point cloud data, depth camera data, and infrared data, etc. Data like this can be called multi-view data. Multi-view clustering (MVC) is the operation of clustering multi-view data.

As an effective data mining method, multi-view clustering has been widely concerned \cite{yun2023low,zhang2023dropping,zhao2023contrastive,lu2024efficient,hao2024ensemble,yang2024discrete,wang2024incomplete}. One of the most representative methods is the graph-based multi-view clustering method \cite{tang2023multi,zhao2023auto,xiao2023adaptive,he2023self,mei2023joint,wang2024joint}. These method involve similarity graphs construct and eigen-decompose of Laplacian matrices, and the computational complexity is $\mathcal{O}(n^2)$ and $\mathcal{O}(n^3)$, respectively, where $n$ represents the number of samples. With the advent of the era of big data, data acquisition is becoming easier, resulting in the continuous expansion of the scale of datasets. So graph-based multi-view clustering methods have been somewhat difficult to deal with large-scale multi-view data. On this basis, anchor graph-based multi-view clustering methods have been proposed and widely used.

The core idea of the anchor graph-based method is to select $m$ representative points (called anchors) from the $n$ samples, and then learn the relationship between the samples and anchors (called anchor graph). The anchor graph is $n \times m$. Generally speaking, $m \ll n$, so the anchor graph-based multi-view clustering methods can handle large-scale multi-view data well.
In addition to being able to handle large-scale data, anchor graph-based methods also inherit the advantages of graph-based methods represented by spectral clustering, \ie, they are not affected by the geometric distribution of data. This method usually first explores the geometrical structure of multi-view data by constructing anchor graph, and then clusters on the anchor graph using existing clustering technics. However, the second step is still time-consuming.

To end this, another method using anchor graph clustering is proposed. Theses methods use anchor graphs to obtain bipartite graphs, and then learns bipartite graphs with $K$ connected components (where $K$ is the number of classes), so that the clustering results can be obtained directly without post-processing. However, these methods of learning bipartite graphs have high parameter requirements and may not find $K$ connected components in some cases.

In order to avoid the above problems and avoid post-processing, we consider the use of projection matrix to project the anchor graph directly into the label space, \ie, the $n \times m$ anchor graph is regarded as feature matrices with $n$ samples and $m$ feature dimensions, and the $n \times c$ cluster label matrix can be directly obtained after the projection transformation. Generally speaking, the above process requires the projection transformation of the anchor graph on each view, and then the clustering label matrix of each view is obtained, finally these matrices are fused to obtain the clustering results of multi-view data. However, projective transformation in each view separately may cause the spatial structure information between different views not fully utilized.

To end this, we propose a multi-view data label learning method based on tensor projection (LLMTP). In order to get the cluster labels directly from the anchor graph, we consider projecting the anchor graph into the label space, so that the cluster results can be obtained directly. Considering that the projection of each view separately will cause the model to be unable to make good use of the complementary information and spatial structure information between different views, we extend the matrix projection transformation of the anchor graph to the tensor projection transformation of the anchor graph tensor, \ie, project the third-order tensor directly. So that the spatial structure information embedded between different views can be preserved to a large extent. Thus, better clustering performance can be obtained. Extensive experiments have proved the superiority of our proposed model.

In summary, the main contributions of this paper are as follows:
\begin{itemize}
  \item It is proposed to project the anchor graph into the label space to obtain the clustering label directly. Meanwhile, the matrix projection is extended to tensor projection so that complementary information and spatial structure information between views can be fully mined.
  \item We propose an algorithm to optimize the tensor projection transformation, and verify the convergence of this algorithm from the experimental point of view.
  \item We introduce the tensor Schatten p-norm to exploit complementary information across different views, facilitating the derivation of a common consensus label matrix.
\end{itemize}

%------------------------------------------------------------------------
\section{Related Works}
\subsection{Anchor Graph-Based Multi-View Clustering Methods}
Since the computational complexity of $n \times n$ similarity graph is $\mathcal O (Vn^2)$ during graph construction, and $\mathcal O (n^3)$ during eigen decomposition of Laplace matrix (where $V$ and $n$ are the number of views and samples, respectively), it is difficult for graph-based multi-view clustering methods to deal with large-scale multi-view data. By using $m$ anchors to cover $n$ sample point clouds, and constructing the relationship between $n$ samples and $m$ anchors, the anchor graph is obtained. Because $m \ll n$, the multi-view clustering method based on anchor graph can deal with large-scale multi-view data effectively.

In general, anchor graph-based methods explore the geometry of multi-view data by constructing anchor graphs, and then it requires additional clustering techniques such as K-means to cluster the obtained anchor graph. However, the additional clustering techniques is time-consuming.

To avoid post-processing, it is common to use anchor graphs to learn bipartite graphs with K connected components (where K is the number of clusters). Therefore, post-processing can be avoided and clustering results can be obtained directly from bipartite graphs with K connected components. The representative methods are as follows.
MVGL \cite{zhan2018graph} gains insights from various single view graphs to construct a global graph. Instead of relying on post-processing techniques, it utilizes K connected components to extract clustering metrics. However, its scalability for handling large-scale data might be limited.
%LMVSC \cite{zhao2020large} builds anchor graphs in each view and then uses a novel integration mechanism to fuse these graphs. It can accelerate eigen decomposition significantly.
SFMC \cite{li2022multiview} introduces a parameter-free approach to fuse multi-view cluster graphs, resulting in cohesive composite graphs through the utilization of self-supervised weighting. Moreover, it employs K connected components to symbolize clusters.
LCBG \cite{zhou2022low} takes the intra-view and inter-view spatial low-rank structures of the learned bipartite graphs into account by minimizing tensor Schatten $p$-norm and nuclear norm.
MSC-BG \cite{yang2022multiview} leverages the Schatten $p$-norm to investigate the synergistic information across diverse views, and deriving clusters through K connecting components.
TBGL \cite{xia2023tensorized} employs the Schatten $p$-norm to delve into the resemblances among different views, while simultaneously integrating $\ell_{\textrm{1,2}}$-norm minimization regularization and connectivity constraints to investigate the similarities within each view.

But these methods of learning bipartite graphs have high parameter requirements and may not find $K$ connected components in some cases.

\subsection{Multi-View Clustering Method Based on Projection}
In the process of multi-view clustering, in most cases, we directly process the original data or the constructed similarity graph, but there may be redundant information, noise or outlier information in the original data or similarity graph. They will adversely affect the final clustering performance. Multi-view clustering methods based on projection are mostly studied to solve such problems.

The general practice of this kind of method is to construct an orthogonal projection matrix to projective the original data, and then get a relatively clean representation matrix in the embedded space.
Gao \etal \cite{gao2020multi} propose a new multi-view clustering framework that combines dimensionality reduction, manifold structure learning and feature selection, which maps high-dimensional data to low-dimensional spaces to reduce data complexity and reduce the effects of data noise and redundancy.
Wang \etal \cite{wang2020robust} proposed a robust self-weighted multi-view projection clustering (RSwMPC) based on $\ell_{\textrm{2,1}}$-norm. It can simultaneously reduce dimension, suppress noise and learn local structure graph. The resulting optimal graph can be directly used for clustering without any other processing.
Sang \etal \cite{sang2022consensus} proposed a consensus graph-based auto-weighted multi-view projection clustering (CGAMPC). It can simultaneously reduce dimension, save manifold structure and learn consensus structure graph. The information similarity graph is constructed on the projected data to ensure the removal of redundant and noisy information in the original similarity graph, and the $\ell_{\textrm{2,1}}$-norm is used to select adaptive discriminant features.
Wang \etal \cite{wang2022clustering} proposed consistency and diversity preserving with projection decomposition for multi-view clustering (CDP2D). It automatically learns the shared projection matrix and analyzes multi-view data through projection matrix decomposition.
Li \etal \cite{li2023projection} propose a projection-based coupled tensor learning method (PCTL). It constructs an orthogonal projection matrix to obtain the main feature information of the raw data of each view, and learns the representation matrix in a clean embedded space. Moreover, tensor learning is used to coupling projection matrix and representation matrix, mining higher-order information between views, and constructing more suitable and better representation of embedded space.

Inspired by the above method, we consider whether the $n \times m$ anchor graph can be directly regarded as an feature matrix with $n$ samples and $m$ dimensional features, according to \cite{wang2022align}. Then the feature matrix is directly projected into the label space through projection changes, and the final clustering label is obtained directly.

%------------------------------------------------------------------------
\section{Notations}\label{Notations}
We will cover t-product and the definition of the tensor Schatten $p$-norm in this section.
\begin{definition}[t-product \cite{kilmer2011factorization}]\label{def:t-prod}
Suppose ${\bm{\mathcal{A}}}\in\mathbb{R}^{n_1\times m\times n_3}$ and ${\bm{\mathcal{B}}}\in \mathbb{R}^{m\times n_2\times n_3}$, the t-product ${\bm{\mathcal{A}}}*{\bm{\mathcal{B}}}\in\mathbb{R}^{n_1\times n_2\times n_3}$ is given by
\begin{align*}
{\bm{\mathcal{A}}}*{\bm{\mathcal{B}}} = \mathrm{ifft}(\mathrm{bdiag}(\overline{\mathbf A}\overline{\mathbf B}),[\ ],3),
\end{align*}
where $\overline{\mathbf A}=\mathrm{bdiag}(\bm{\overline{\mathcal{A}}})$ and it denotes the block diagonal matrix. The blocks of $\overline{\mathbf A}$ are frontal slices of $\bm{\overline{\mathcal{A}}}$.
\end{definition}

\begin{definition}\label{tensorSpNorm}~\cite{gao2020enhanced}
Given ${\bm{\mathcal H}}\in{\mathbb{R}}^{n_1 \times n_2 \times n_3}$, $h=\min(n_1,n_2)$, the tensor Schatten $p$-norm of  ${\bm{\mathcal H}}$ is defined as
\begin{equation}
\begin{array}{c}
{\left\| {\bm{\mathcal H}} \right\|_{{\Sp}}} = {\left( {\sum\limits_{i = 1}^{{n_3}} {\left\| {{{\bm{\overline {\mathcal H} }^{(i)}}}} \right\|_{{\Sp}}^p} } \right)^{\frac{1}{p}}} = {\left( {\sum\limits_{i = 1}^{{n_3}} {\sum\limits_{j = 1}^h {{\sigma _j}{{\left( {{\bm{\overline {\mathcal H} }^{(i)}}} \right)}^p}} } } \right)^{\frac{1}{p}}},
\end{array}
\label{4}
\end{equation}
where, $0 < p \leqslant 1$, ${\sigma _j}(\overline{\bm{\mathcal H}}^{(i)})$ denotes the j-th singular value of $\overline{\bm{\mathcal H}}^{(i)}$.
\label{definition1}
\end{definition}
It should be pointed out that  for $0 < p \leqslant 1$ when $p$ is appropriately chosen, the  Schatten $p$-norm provides quite effective improvements for a tighter approximation of the rank function~\cite{zha2020benchmark,xie2016weighted}.

Also we introduce the notations used throughout this paper. We use bold calligraphy letters for 3rd-order tensors, ${\bm{\mathcal {H}}} \in{\mathbb{R}} {^{{n_1} \times {n_2} \times {n_3}}}$, bold upper case letters for matrices, ${\mathbf{H}}$, bold lower case letters for vectors, ${\bf{h}}$, and lower case letters such as ${h_{ijk}}$ for the entries of ${\bm{\mathcal {H}}}$. Moreover, the $i$-th frontal slice of ${\bm{\mathcal {H}}}$ is ${\bm{\mathcal {H}}}^{(i)}$. $\overline {{\bm{\mathcal {H}}}}$ is the discrete Fourier transform (DFT) of ${\bm{\mathcal {H}}}$ along the third dimension, $\overline {{\bm{\mathcal {H}}}} = \mathrm{fft}({{\bm{\mathcal H}}},[\ ],3)$. Thus, $\bm{{\mathcal H}} = \mathrm{ifft}({\overline {\bm{\mathcal H}}},[\ ],3)$. The trace and transpose of matrix $\mathbf{H}$ are expressed as $\mathrm{tr}(\mathbf{H})$ and  $\mathbf{H}^{\mathrm{T}}$. The F-norm of ${\bm{\mathcal H}}$ is denoted by ${\left\| {\bm{\mathcal H}}\right\|_F}$.

%------------------------------------------------------------------------
\section{Methodology}
\subsection{Motivation and Objective Function}
Multi-view clustering method based on anchor graph can deal with large-scale multi-view data efficiently. By learning a bipartite graph with K connected components using anchor graphs, the cluster labels can be obtained directly without any post-processing. However, this method has high requirements on parameters, and in some cases it may not be able to obtain a clear bipartite graph with K connected components.

Inspired by the excellent performance of the orthogonal projection matrix in processing redundant information in the raw data matrix, and referring to \cite{wang2022align}, we consider the $n\times m$ anchor graph as an feature matrix composed of $n$ $m$-dimensional feature vectors and the $n\times c$ cluster label matrix can be obtained directly by projecting the feature matrix (\ie anchor graph) into the label space (as shown in Figure \ref{projection}). In this way, the clustering results can be obtained directly without post-processing. When dealing with multi-view data we have the following formula:
\begin{equation}\label{f1}
    \begin{aligned}
        &\min \sum_{v=1}^{V} {\left\| {\mathbf{S}^{(v)}}  {\mathbf{G}^{(v)}} - {\mathbf{H}^{(v)}}  \right \|}_F^2 + \lambda \sum_{v=1}^{V} \mathcal R (\mathbf H^{(v)})  \\
        {\textrm{s.t.}} \quad  &{\mathbf{H}^{(v)}}\geqslant 0, {\mathbf{H}^{(v)}}^{\mathrm{T}}{\mathbf{H}^{(v)}} = {\mathbf{I}}, {\mathbf{G}^{(v)}}^{\mathrm{T}} {\mathbf{G}^{(v)}} = {\mathbf{I}},
    \end{aligned}
\end{equation}
where, ${\mathbf{S}^{(v)}}$ is the anchor graph from the $v$-th view, the anchors selection method and the anchor graph construction method are the same as that of \cite{xia2023tensorized}. ${\mathbf{G}^{(v)}}$ is the orthogonal projection matrix. ${\mathbf{H}^{(v)}}$ is the cluster label matrix. It can be more interpretable by constraining it to be non-negative and orthogonal, \ie, each row of ${\mathbf{H}^{(v)}}$ has only one non-zero value, and the position of the value indicates the cluster to which the corresponding element of the row belongs.
The purpose of the regular term with coefficient of $\lambda$ is to make the clustering label matrices of different views tend to be consistent.

\begin{figure}[htbp]
	\centering
	\includegraphics[width=1.0\linewidth]{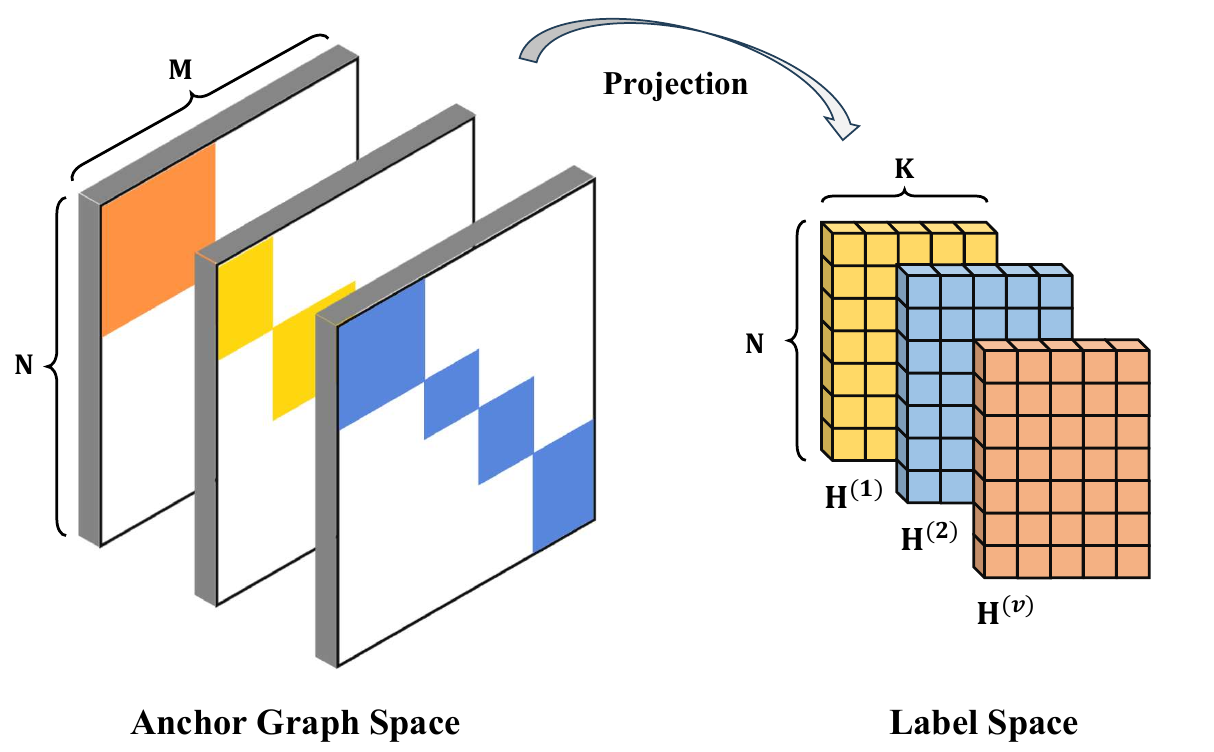}
	\caption{Anchor graph space to label space}
	\label{projection}
\end{figure}

Considering that the projection transformation in Eq. (\ref{f1}) is carried out separately in each view, the complementary information and spatial structure information embedded in different views may not be fully utilized. To end this, we consider extending the two-dimensional matrix projection into a third-order tensor projection.
Eq. (\ref{f1}) is extended as follows:
\begin{equation}\label{f2}
    \begin{aligned}
        &\min {\left\| \bm{\mathcal S}*\bm{\mathcal G} - \bm{\mathcal H} \right \|}_F^2 + \lambda \sum_{v=1}^{V} \mathcal R (\mathbf H^{(v)}) \\
        {\textrm{s.t.}} \quad  &\bm{\mathcal H}\geqslant 0, \bm{\mathcal H}^{\mathrm{T}}*\bm{\mathcal H} = \bm{\mathcal I}, \bm{\mathcal G}^{\mathrm{T}}*\bm{\mathcal G} = \bm{\mathcal I}
    \end{aligned}
\end{equation}

In multi-view clustering, we should try our best to make the ${\mathbf{H}^{(v)}}$ of different views in (\ref{f1}) tend to be the same. Inspired by the excellent performance of the tensor Schatten $p$-norm \cite{gao2020enhanced,yang2022multiview,xia2023tensorized,li2023orthogonal}, we fully explore the complementary information in label matrices of different views by introducing the regular term of the tensor Schatten $p$-norm.

The final objective function is as follows:

\begin{equation}\label{of}
    \begin{aligned}
        &\min {\left\| \bm{\mathcal S}*\bm{\mathcal G} - \bm{\mathcal H} \right \|}_F^2 + \lambda {\left \|{\bm{\mathcal H}}\right \|}_\Sp^p \\
        {\textrm{s.t.}} \quad  &\bm{\mathcal H}\geqslant 0, \bm{\mathcal H}^{\mathrm{T}}*\bm{\mathcal H} = \bm{\mathcal I}, \bm{\mathcal G}^{\mathrm{T}}*\bm{\mathcal G} = \bm{\mathcal I}
    \end{aligned}
\end{equation}
where $0<p \leqslant 1$, $\lambda$ is the hyper-parameter of the Schatten $p$-norm term. The construction process of the tensor ${\bm{\mathcal H}}$ is shown in Figure \ref{tensor}. Remark 1 briefly describes the role of the tensor Schatten $p$-norm.

\begin{remark}[Explanation of the tensor Schatten $p$-norm]\label{rTensorSp}
    For the tensor $\bm{\mathcal{H}}$, as depicted in Fig \ref{tensor}, its $k$-th lateral slice $\Theta^k$ represents the relationship of $n$ samples with the $k$-th cluster across different views. Multi-view clustering aims to harmonize the sample-cluster relationships in different views, making $\mathbf{H}_{:,k}^{(1)}, \cdots, \mathbf{H}_{:,k}^{(v)}$ as congruent as possible. However, the clustering structures often vary significantly across views. Applying the tensor Schatten $p$-norm to $\bm{\mathcal{H}}$ ensures that $\Phi^k$ maintains a spatially low-rank structure, leveraging complementary information across views and fostering consistency in the clustering labels.
\end{remark}

\begin{figure}[htbp]
	\centering
	\includegraphics[width=1.0\linewidth]{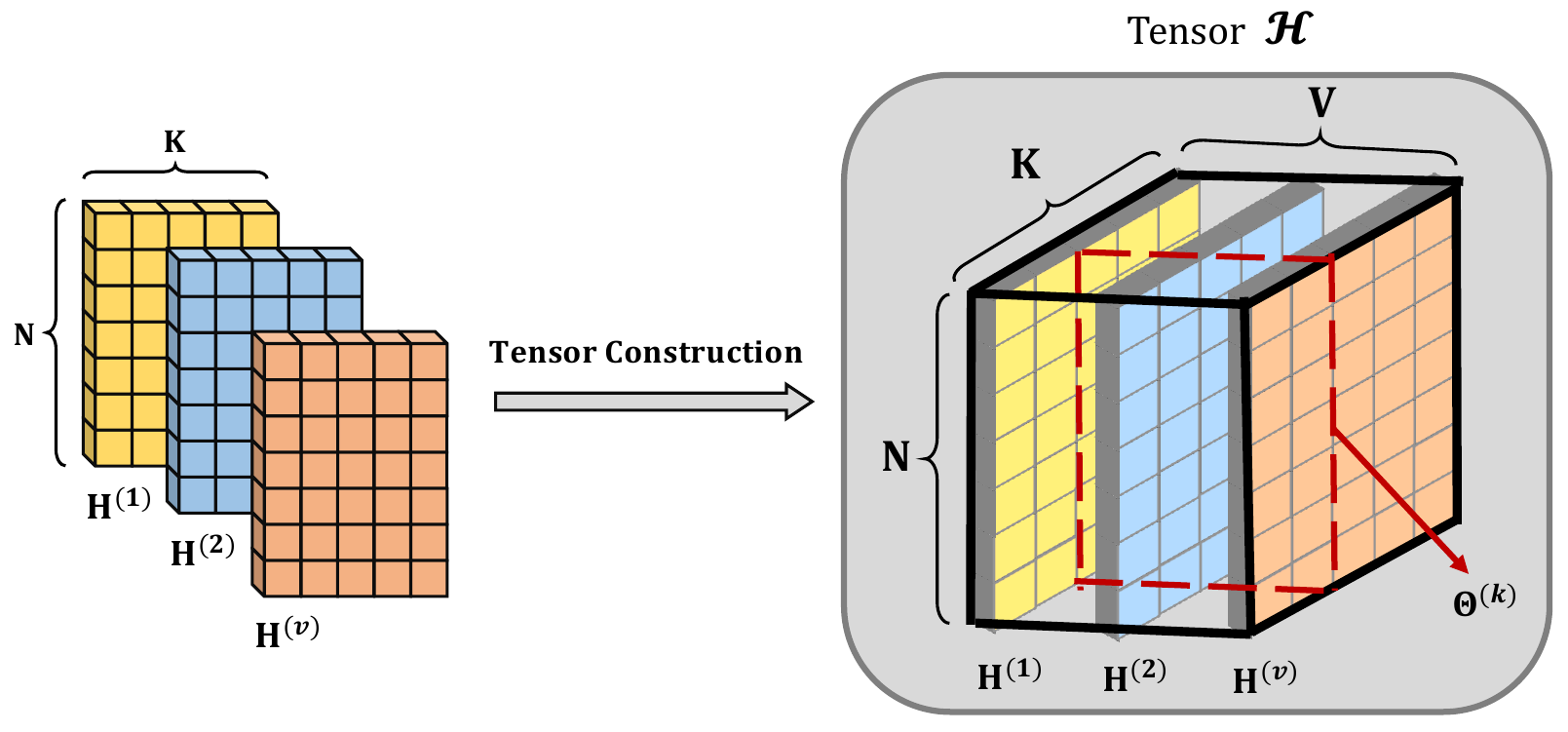}
	\caption{Tensor construction}
	\label{tensor}
\end{figure}

\subsection{Optimization}
Inspired by Augmented Lagrange Multiplier (ALM), we introduce two auxiliary variables $\bm{\mathcal Q}$ and $\bm{\mathcal J}$ and let $\bm{\mathcal H} = \bm{\mathcal Q}$, $\bm{\mathcal H} = \bm{\mathcal J}$, respectively, where $\bm{\mathcal Q}\geqslant 0$. Then, we rewrite the model as the  following unconstrained problem:
\begin{equation}\label{objective function}
	\begin{aligned}
		&\min {\left\| \bm{\mathcal S}*\bm{\mathcal G} - \bm{\mathcal H} \right \|}_F^2 + \lambda{\| \bm{\mathcal J}\|}_\Sp^p \\
        & + \frac{\mu}{2} {\left\| \bm{\mathcal H}-\bm{\mathcal Q} + \frac{\bm{\mathcal Y_1}}{\mu}\right \|}_F^2 + \frac{\rho}{2} {\left\| \bm{\mathcal H}-\bm{\mathcal J} + \frac{\bm{\mathcal Y_2}}{\rho}\right \|}_F^2 \\
   {\textrm{s.t.}} \quad  &\bm{\mathcal Q}\geqslant 0, \bm{\mathcal H}^{\mathrm{T}}*\bm{\mathcal H} = \bm{\mathcal I}, \bm{\mathcal G}^{\mathrm{T}}*\bm{\mathcal G} = \bm{\mathcal I}
	\end{aligned}
\end{equation}
where $\bm{\mathcal Y_1}$, $\bm{\mathcal Y_2}$ represent Lagrange multipliers and $\mu$, $\rho$ are the penalty parameters. The optimization process can therefore be separated into four steps:

$\bullet$\textbf{Solve $\bm{\mathcal G}$ with fixed $\bm{\mathcal Q}, \bm{\mathcal H}, \bm{\mathcal J}$.}  (\ref{objective function}) becomes:
\begin{equation}\label{solveG1}
	\begin{aligned}
		\min_{\bm{\mathcal G}^{\mathrm{T}}*\bm{\mathcal G} = \bm{\mathcal I}} {\left\| \bm{\mathcal S}*\bm{\mathcal G} - \bm{\mathcal H} \right \|}_F^2 ,
	\end{aligned}
\end{equation}

(\ref{solveG1}) is equivalent to the following in the frequency domain:
\begin{equation}\label{solveG2}
    \begin{aligned}
        \min_{(\bm{\mathcal {\overline{G}}}^{(v)})^{\mathrm{T}} \bm{\mathcal {\overline{G}}}^{(v)} = \mathbf I}
        &\sum_{v=1}^{V}{\left\| \bm{\mathcal {\overline{S}}}^{(v)} \bm{\mathcal {\overline{G}}}^{(v)} - \bm{\mathcal {\overline{H}}}^{(v)}  \right \|}_F^2,
    \end{aligned}
\end{equation}
where $\bm{\mathcal {\overline{G}}} = \mathrm{fft}({{\bm{\mathcal G}}},[\ ],3)$, and the others in the same way.

(\ref{solveG2}) can be reduced to:
\begin{equation}\label{solveG3}
    \begin{aligned}
        \min_{(\bm{\mathcal {\overline{G}}}^{(v)})^{\mathrm{T}} \bm{\mathcal {\overline{G}}}^{(v)} = \mathbf I}
        &\mathrm{tr} \left( (\bm{\mathcal {\overline{G}}}^{(v)})^{\mathrm{T}} (\bm{\mathcal {\overline{S}}}^{(v)})^{\mathrm{T}}\bm{\mathcal {\overline{S}}}^{(v)} \bm{\mathcal {\overline{G}}}^{(v)}  \right) \\
        &- 2\mathrm{tr} \left( (\bm{\mathcal {\overline{G}}}^{(v)})^{\mathrm{T}} (\bm{\mathcal {\overline{S}}}^{(v)})^{\mathrm{T}} \bm{\mathcal {\overline{H}}}^{(v)} \right),
    \end{aligned}
\end{equation}

(\ref{solveG3}) is equivalent to:
\begin{equation}\label{solveG4}
    \begin{aligned}
        \max_{(\bm{\mathcal {\overline{G}}}^{(v)})^{\mathrm{T}} \bm{\mathcal {\overline{G}}}^{(v)} = \mathbf I}
        &\mathrm{tr} \left( (\bm{\mathcal {\overline{G}}}^{(v)})^{\mathrm{T}} \bm{\mathcal {\overline{W}}}_1^{(v)} \bm{\mathcal {\overline{G}}}^{(v)}  \right) + 2\mathrm{tr} \left( (\bm{\mathcal {\overline{G}}}^{(v)})^{\mathrm{T}} \bm{\mathcal {\overline{W}}}_2^{(v)} \right),
    \end{aligned}
\end{equation}
where $\bm{\mathcal {\overline{W}}}_1^{(v)} = \beta \mathbf I - (\bm{\mathcal {\overline{S}}}^{(v)})^{\mathrm{T}}\bm{\mathcal {\overline{S}}}^{(v)} $ and $\bm{\mathcal {\overline{W}}}_2^{(v)} =  (\bm{\mathcal {\overline{S}}}^{(v)})^{\mathrm{T}} \bm{\mathcal {\overline{H}}}^{(v)} $, where $\beta$ is an arbitrary constant to ensure that $\bm{\mathcal {\overline{W}}}_1^{(v)}$ is a positive definite matrix.

To solve (\ref{solveG4}), we introduce the following Theorem:
\begin{theorem}\label{theorem solveG}
\cite{Xu2020low} For the model:
\begin{equation}\label{the1}
    \max_{ \mathbf G^{\mathrm{T}} \mathbf G =\mathbf I} \mathrm{tr}( \mathbf G^{\mathrm{T}} \mathbf B \mathbf G) + 2 \mathrm{tr}(\mathbf G^{\mathrm{T}} \mathbf K)
\end{equation}
$\mathbf G$ is solved iteratively and $\mathbf G^\ast=\mathbf U \mathbf V^{\mathrm{T}}$, where $\mathbf U$, $\mathbf V$ is from the SVD decomposition: $\mathbf U \mathbf X \mathbf V^{\mathrm{T}} = \mathbf B \mathbf G + \mathbf K$.
\end{theorem}

According to Theorem \ref{theorem solveG}, the $\bm{\mathcal {\overline{G}}}^{(v)}$ can be solved iteratively
and the solution is:
\begin{equation}\label{solveG}
    \bm{\mathcal {\overline{G}}}^{\ast(v)} = \bm{\mathcal {\overline{U}}}^{(v)} (\bm{\mathcal {\overline{V}}}^{(v)})^{\mathrm{T}},
\end{equation}
where $\bm{\mathcal {\overline{U}}}^{(v)} \bm{\mathcal {\overline{X}}}^{(v)} (\bm{\mathcal {\overline{V}}}^{(v)})^{\mathrm{T}} = \bm{\mathcal {\overline{W}}}_1^{(v)} \bm{\mathcal {\overline{G}}}^{(v)} + \bm{\mathcal {\overline{W}}}_2^{(v)}$.

$\bullet$\textbf{Solve $\bm{\mathcal H}$ with fixed $\bm{\mathcal Q}, \bm{\mathcal G}, \bm{\mathcal J}$.}  (\ref{objective function}) becomes:
\begin{equation}\label{solveH1}
	\begin{aligned}
		\min_{\bm{\mathcal H}^{\mathrm{T}}*\bm{\mathcal H} = \bm{\mathcal I}} {\left\| \bm{\mathcal S}*\bm{\mathcal G} - \bm{\mathcal H} \right \|}_F^2
        &+ \frac{\mu}{2} {\left\| \bm{\mathcal H}-\bm{\mathcal Q} + \frac{\bm{\mathcal Y_1}}{\mu}\right \|}_F^2 \\
        &+ \frac{\rho}{2} {\left\| \bm{\mathcal H}-\bm{\mathcal J} + \frac{\bm{\mathcal Y_2}}{\rho}\right \|}_F^2,
	\end{aligned}
\end{equation}

(\ref{solveH1}) is equivalent to the following in the frequency domain:
\begin{equation}\label{solveH2}
    \begin{aligned}
        \min_{(\bm{\mathcal {\overline{H}}}^{(v)})^{\mathrm{T}} \bm{\mathcal {\overline{H}}}^{(v)} = \mathbf I}
        &\sum_{v=1}^{V}{\left\| \bm{\mathcal {\overline{S}}}^{(v)} \bm{\mathcal {\overline{G}}}^{(v)} - \bm{\mathcal {\overline{H}}}^{(v)}  \right \|}_F^2 \\
        & + \sum_{v=1}^{V} \frac{\mu}{2} {\left\| \bm{\mathcal {\overline{H}}}^{(v)} - \bm{\mathcal {\overline{Q}}}^{(v)} + \frac{\bm{\mathcal {\overline{Y}}}_1^{(v)}}{\mu}\right \|}_F^2 \\
        & + \sum_{v=1}^{V} \frac{\rho}{2} {\left\| \bm{\mathcal {\overline{H}}}^{(v)} - \bm{\mathcal {\overline{J}}}^{(v)} + \frac{\bm{\mathcal {\overline{Y}}}_2^{(v)}}{\rho}\right \|}_F^2,
    \end{aligned}
\end{equation}
where $\bm{\mathcal {\overline{H}}} = \mathrm{fft}({{\bm{\mathcal H}}},[\ ],3)$, and the others in the same way.

(\ref{solveH2}) can be reduced to:
\begin{equation}\label{solveH3}
    \begin{aligned}
        \min_{(\bm{\mathcal {\overline{H}}}^{(v)})^{\mathrm{T}} \bm{\mathcal {\overline{H}}}^{(v)} = \mathbf I}
        &-2\mathrm{tr} \left( (\bm{\mathcal {\overline{H}}}^{(v)})^{\mathrm{T}} \bm{\mathcal {\overline{S}}}^{(v)} \bm{\mathcal {\overline{G}}}^{(v)}  \right) - \mu \mathrm{tr} \left( (\bm{\mathcal {\overline{H}}}^{(v)})^{\mathrm{T}} \bm{\mathcal {\overline{W}}}_3^{(v)} \right) \\
        & - \rho \mathrm{tr} \left( (\bm{\mathcal {\overline{H}}}^{(v)})^{\mathrm{T}} \bm{\mathcal {\overline{W}}}_4^{(v)} \right)
    \end{aligned}
\end{equation}
where $\bm{\mathcal {\overline{W}}}_3^{(v)} =  \bm{\mathcal {\overline{Q}}}^{(v)} - \frac{ \bm{\mathcal {\overline{Y}}}^{(v)}_1}{\mu}$ and $\bm{\mathcal {\overline{W}}}_4^{(v)} =  \bm{\mathcal {\overline{J}}}^{(v)} - \frac{ \bm{\mathcal {\overline{Y}}}^{(v)}_2}{\rho}$.

(\ref{solveH3}) can be reduced to:
\begin{equation}\label{solveH4}
    \max_{(\bm{\mathcal {\overline{H}}}^{(v)})^{\mathrm{T}} \bm{\mathcal {\overline{H}}}^{(v)} = \mathbf I}
    \mathrm{tr} \left( (\bm{\mathcal {\overline{H}}}^{(v)})^{\mathrm{T}} \bm{\mathcal {\overline{A}}}^{(v)} \right)
\end{equation}
where $\bm{\mathcal {\overline{A}}}^{(v)} = 2 \bm{\mathcal {\overline{S}}}^{(v)} \bm{\mathcal {\overline{G}}}^{(v)} + \mu \bm{\mathcal {\overline{W}}}_3^{(v)} + \rho \bm{\mathcal {\overline{W}}}_4^{(v)}$.

To solve (\ref{solveH4}), we introduce the following Theorem:
\begin{theorem}\label{theorem solveF}
Given $\mathbf G$ and $\mathbf P$, where $\mathbf G (\mathbf G)^{\mathrm{T}}=\mathbf I$ and $\mathbf P$ has the singular value decomposition $\mathbf P=\mathbf \Lambda \mathbf S(\mathbf V)^{\mathrm{T}}$, then the optimal solution of
\begin{equation}\label{the2}
\max_{\mathbf G (\mathbf G)^{\mathrm{T}}=\mathbf I} \mathrm{tr}( \mathbf G \mathbf P)
\end{equation}
is $\mathbf G^\ast=\mathbf V[\mathbf I,\mathbf 0](\mathbf \Lambda)^{\mathrm{T}}$.
\end{theorem}

\begin{proof}
From the SVD $\mathbf P=\mathbf \Lambda \mathbf S (\mathbf V)^{\mathrm{T}}$ and together with  (\ref{the1}), it is evident that
\begin{equation}\label{solveF5}
\begin{aligned}
		\mathrm{tr}(\mathbf G \mathbf P ) &= \mathrm{tr}(\mathbf G\mathbf \Lambda \mathbf S(\mathbf V)^{\mathrm{T}}  ) \\
                                          &=\mathrm{tr}(\mathbf S (\mathbf V)^{\mathrm{T}} \mathbf G \mathbf \Lambda  ) \\
		                                  &=\mathrm{tr}(\mathbf S \mathbf H )  \\
                                          &=\sum_i s_{ii}  h_{ii}
	,\end{aligned}
\end{equation}
where $\mathbf H=(\mathbf V)^{\mathrm{T}} \mathbf G \mathbf \Lambda$, $s_{ii}$ and $h_{ii}$ are the $(i,i)$ elements of $\mathbf S$ and $\mathbf H$, respectively. It can be easily verified that $\mathbf H (\mathbf H)^{\mathrm{T}}=\mathbf I$, where $\mathbf I$ is an identity matrix. Therefore $-1\leqslant h_{ii} \leqslant 1$ and $s_{ii} \geqslant 0$, Thus we have:
\begin{equation}\label{solveF6}
\mathrm{tr}(\mathbf G \mathbf P)=\sum_i s_{ii} h_{ii} \leqslant \sum_i s_{ii}.
\end{equation}
The equality holds when $\mathbf H$ is an identity matrix. $\mathrm{tr}(\mathbf G \mathbf P)$ reaches the maximum when $\mathbf H = [\mathbf I,\mathbf 0]$.
\end{proof}

According to Theorem \ref{theorem solveF} the solution of (\ref{solveH4}) is:
\begin{equation}\label{solveH}
    \bm{\mathcal {\overline{H}}}^{(v)} = \bm{{\overline{\Lambda}}}^{(v)} (\bm{{\overline{V}}}^{(v)})^{\mathrm{T}}
\end{equation}
where $\bm{{\overline{\Lambda}}}^{(v)}$ and $\bm{{\overline{V}}}^{(v)}$ can be obtained by SVD $\bm{\mathcal {\overline{A}}}^{(v)}=\bm{{\overline{\Lambda}}}^{(v)} \mathbf X (\bm{{\overline{V}}}^{(v)})^{\mathrm{T}}$

$\bullet$\textbf{Solve $\bm{\mathcal Q}$ with fixed $\bm{\mathcal H}, \bm{\mathcal G}, \bm{\mathcal J}$.}  (\ref{objective function}) becomes:
\begin{equation}\label{solveQ1}
	\begin{aligned}
		\min_{\bm{\mathcal Q}\geqslant 0} = \frac{\mu}{2} {\left\| \bm{\mathcal H}-\bm{\mathcal Q} + \frac{\bm{\mathcal Y_1}}{\mu}\right \|}_F^2,
	\end{aligned}
\end{equation}

(\ref{solveQ1}) is obviously equivalent to:
\begin{equation}\label{solveQ2}
	\begin{aligned}
		\min_{\bm{\mathcal Q}\geqslant 0}
        \frac{\mu}{2} {\left\| (\bm{\mathcal H} + \frac{\bm{\mathcal Y}_1}{\mu}) - \bm{\mathcal {Q}} \right \|}_F^2
	\end{aligned}
\end{equation}

According to \cite{yang2021fast}, the solution of  (\ref{solveQ2}) is:
\begin{equation}\label{solveQ}
    \bm{\mathcal {Q}} = \left(\bm{\mathcal H} + \frac{\bm{\mathcal Y}_1}{\mu} \right)_+
\end{equation}

$\bullet$\textbf{Solve $\bm{\mathcal J}$ with fixed $\bm{\mathcal H}, \bm{\mathcal G}, \bm{\mathcal Q}$.}  (\ref{objective function}) becomes:
\begin{equation}\label{solveJ1}
	\begin{aligned}
		\min = \lambda{\| \bm{\mathcal J}\|}_\Sp^p + \frac{\rho}{2} {\left\| \bm{\mathcal H}-\bm{\mathcal J} + \frac{\bm{\mathcal Y_2}}{\rho}\right \|}_F^2,
	\end{aligned}
\end{equation}

We can deduce
\begin{equation}\label{solveJ}
	\begin{aligned}
		\bm{\mathcal J}^* = \arg \min \frac{1}{2}\left\|{\bm{\mathcal H} + \frac{\bm{\mathcal Y}_2}{\rho} - \bm{\mathcal J}}\right\|_F^2 + \frac{\lambda}{\rho}{\|\bm{\mathcal J}\|}_\Sp^p,
	\end{aligned}
\end{equation}
which has a closed-form solution as Lemma \ref{T2} \cite{gao2020enhanced}:

\begin{lemma}\label{T2}
    Let ${\mathcal Z} \in {\mathbb{R}}{^{{n_1} \times {n_2} \times {n_3}}}$ have a t-SVD ${\mathcal Z} = {\mathcal U} * {\mathcal S} * {{\mathcal V}^{\mathrm{T}}}$, then the optimal solution for
    \begin{equation}\label{tensor-gaozx-2020}
        \begin{array}{l}
            \mathop {\min }\limits_{\mathcal X} \frac{1}{2}\left\| {{\mathcal X} - {\mathcal Z}} \right\|_F^2 + \tau \left\| {\mathcal X} \right\|_{{\Sp}}^p.
        \end{array}
    \end{equation}
is ${{\mathcal X}^*} = {\Gamma _\tau }({\mathcal Z}) = {\mathcal U}*\mathrm{ifft}({P_\tau }(\overline {\mathcal Z} ))*{{\mathcal V}^{\mathrm{T}}}$, where ${P_\tau }(\overline {\mathcal Z} )$ is an f-diagonal  3rd-order tensor, whose diagonal elements can be found by using the GST algorithm introduced in \cite{gao2020enhanced}.
\end{lemma}

The solution of  (\ref{solveJ}) is:
\begin{equation}\label{21}
    \bm{\mathcal J}^* = {\Gamma _{\frac{\lambda}{\rho}}} (\bm{\mathcal H} + \frac{\bm{\mathcal Y}_2}{\rho}).
\end{equation}

Finally, the optimization procedure for Label Learning Method Based on Tensor Projection (LLMTP) is outlined in Algorithm \ref{A1}.
\begin{algorithm}[tb]
\caption{Label Learning Method Based on Tensor Projection (LLMTP)}
\label{A1}
%\textbf{Input}: Data matrices $\{{\mathbf{X}}^{(v)}\}_{v=1}^{V}\in \mathbb{R}^{N\times d_v}$; anchors numbers $m$; cluster number $K$.\\
%\textbf{Output}: Cluster labels $\mathbf{Y}$ of each data points.\\
%\textbf{Initialize}: $\alpha_v=1/V$, $\mu=10^{-5}$, $\rho=10^{-5}$, $\eta=1.1$, $\bm{\mathcal Y}_1=0$, ${{\bf{Y}}_2^{(v)}} = 0$, and $\mathbf M^{(v)}$ is an identity matrix;
\begin{algorithmic}[1] %[1] enables line numbers
\REQUIRE Data matrices $\{{\mathbf{X}}^{(v)}\}_{v=1}^{V}\in \mathbb{R}^{N\times d_v}$; anchors numbers $m$; cluster number $K$.
\ENSURE Cluster labels $\mathbf{Y}$ of each data points.
\STATE \textbf{Initialize}: $\mu=10^{-5}$, $\rho=10^{-5}$, $\eta=1.5$, $\bm{\mathcal Y}_1=0$, $\bm{\mathcal Y}_2=0$, $\mathbf{\overline{Q}}^{(v)}$ is identity matrix;
\STATE Compute graph matrix $\mathbf S^{(v)}$ of each views;
\WHILE{not condition}
\STATE Update $\bm{\mathcal {\overline{G}}}^{(v)}$ by solving  (\ref{solveG});
\STATE Update $\bm{\mathcal {\overline{H}}}^{(v)}$ by solving  (\ref{solveQ});
\STATE Update $\bm{\mathcal {\overline{Q}}}^{(v)}$ by solving  (\ref{solveH});
\STATE Update ${\bm{{\mathcal J}}}$ by using  (\ref{solveJ});
\STATE Update $\bm{\mathcal Y}_1$, $\bm{\mathcal Y}_2$, $\mu$ and $\rho$: $\bm{\mathcal Y}_1=\bm{\mathcal Y}_1+\mu(\bm{\mathcal H}-\bm{\mathcal Q})$, $\bm{\mathcal Y}_2=\bm{\mathcal Y}_2+\mu(\bm{\mathcal H}-\bm{\mathcal J})$, $\mu=\min(\eta\mu, 10^{13})$, $\rho=\min(\eta\rho, 10^{13})$;
\ENDWHILE
\STATE Calculate the $K$ clusters by using \\
$\mathbf H=\sum_{v=1}^V \mathbf H^{(v)} / V$;
\STATE \textbf{return} Clustering result (The position of the largest element in each row of the indicator matrix is the label of the corresponding sample).
\end{algorithmic}
\end{algorithm}

%------------------------------------------------------------------------
\section{Experiments}
In this section, we demonstrate the performance of our proposed method through extensive experiments. We evaluate the clustering performance by 3 metrics used widely, \ie, 1) ACC; 2) NMI; 3) Purity. The higher the value the better the clustering results for all metrics mentioned above. To ensure reliability, we conducted 10 independent trials for each method, recording the mean and variance of the results.
Experiments using the MSRC, HandWritten4, Mnist4, and Scene15 datasets were conducted on a laptop equipped with an Intel Core i5-8300H CPU and 16 GB RAM, utilizing Matlab R2018b. In contrast, the Reuters and NoisyMnist were processed on a standard Windows 10 server, featuring dual Intel(R) Xeon(R) Gold 6230 CPUs at 2.10 GHz and 128 GB RAM, with MATLAB R2020a.

The datasets employed in our experiments are detailed in Table \ref{datasets}.
We compare our approach against the following state-of-the-art methods: CSMSC \cite{luo2018consistent}, GMC \cite{wang2019gmc}, ETLMSC \cite{WuLZ19}, LMVSC \cite{kang2020large}, FMCNOF \cite{yang2021fast}, SFMC \cite{li2022multiview}, and FPMVS-CAG \cite{wang2021fast}.

\begin{table}[t]
\caption{Multi-view datasets used in our experiments}
\label{datasets}
\centering
\resizebox{\columnwidth}{!}
{
\begin{tabular}{ccccc}
\toprule[2pt]
\#Dataset 	& \#Samples & \#View & \#Class 	& \multicolumn{1}{l}{\#Feature} \\
\midrule
MSRC    & 210  & 5 & 7 & \multicolumn{1}{l}{24, 576, 512, 256, 254} \\
HandWritten4   & 2000 & 4 & 10 & \multicolumn{1}{l}{76, 216, 47, 6} \\
Mnist4	& 4000 	& 3	& 4	& \multicolumn{1}{l}{30, 9, 30}   \\
Scene15 & 4485 	& 3	& 15	& \multicolumn{1}{l}{1800, 1180, 1240}   \\
Reuters & 18758 & 5 & 6 & \multicolumn{1}{l}{21531, 24892, 34251, 15506, 11547} \\
%Noisy MNIST	& 50000 	& 2	& 10	& \multicolumn{1}{l}{784, 784}   \\
\toprule[2pt]
\end{tabular}
}
\end{table}

\subsection{Clustering Performance}
Table \ref{result1} and Table \ref{result2} show the clustering performance of the proposed model on different datasets. The optimal performance is indicated by \textbf{bold}, and the sub-optimal performance is indicated by \underline{underline}. It can be seen from the three tables that the method proposed in this paper is superior to other comparison methods. Among comparison methods, ETLMSC's performance is relatively suboptimal. It is worth mentioning that this method is a tensor-based spectral clustering method, while others are non-tensor methods. It may means that tensor-based methods have a certain degree of performance improvement compared with non-tensor methods.

\begin{table*}[h]
\caption{Clustering performance on MSRC, HandWritten4, Mnist4 and Scene15}
\label{result1}
\centering
%\resizebox{\textwidth}{!}
%{
\begin{tabular}{c|ccc|ccc}
\toprule[2pt]
Datasets    &\multicolumn{3}{c}{MSRC}  &\multicolumn{3}{c}{HandWritten4} \\
\midrule
Metrices  &ACC &NMI &Purity  &ACC &NMI &Purity  \\
\midrule[1pt]
CSMSC   &0.758$\pm$0.007 &0.735$\pm$0.010 &0.793$\pm$0.008      &0.806$\pm$0.001 &0.793$\pm$0.001 &0.867$\pm$0.001  \\
GMC     &0.895$\pm$0.000 &0.809$\pm$0.000 &0.895$\pm$0.000      &0.861$\pm$0.000 &0.859$\pm$0.000 &0.861$\pm$0.000  \\
ETLMSC  &\underline{0.962$\pm$0.000} &\underline{0.937$\pm$0.000} &\underline{0.962$\pm$0.000}      &\underline{0.938$\pm$0.001} &\underline{0.893$\pm$0.001} &\underline{0.938$\pm$0.001}   \\
LMVSC   &0.814$\pm$0.000 &0.717$\pm$0.000 &0.814$\pm$0.000      &0.904$\pm$0.000 &0.831$\pm$0.000 &0.904$\pm$0.000  \\
FMCNOF  &0.440$\pm$0.039 &0.345$\pm$0.046 &0.449$\pm$0.042      &0.385$\pm$0.092 &0.370$\pm$0.092 &0.386$\pm$0.090   \\
SFMC    &0.810$\pm$0.000 &0.721$\pm$0.000 &0.810$\pm$0.000      &0.853$\pm$0.000 &0.871$\pm$0.000 &0.873$\pm$0.000   \\
FPMVS-CAG &0.786$\pm$0.000 &0.686$\pm$0.000 &0.786$\pm$0.000    &0.744$\pm$0.000 &0.753$\pm$0.000 &0.744$\pm$0.000   \\
Ours    &\textbf{0.986$\pm$0.000} &\textbf{0.971$\pm$0.000} &\textbf{0.986$\pm$0.000}  &\textbf{0.963$\pm$0.000} &\textbf{0.937$\pm$0.000} &\textbf{0.963$\pm$0.000}    \\
\midrule[2pt]
Datasets    &\multicolumn{3}{c}{Mnist4}  &\multicolumn{3}{c}{Scene15}\\
\midrule
Metrices  &ACC &NMI &Purity  &ACC &NMI &Purity   \\
\midrule[1pt]
CSMSC     &0.641$\pm$0.000 &0.601$\pm$0.010 &0.728$\pm$0.008    &0.334$\pm$0.008 &0.313$\pm$0.005 &0.378$\pm$0.003 \\
GMC       &0.920$\pm$0.000 &0.807$\pm$0.000 &0.920$\pm$0.000    &0.140$\pm$0.000 &0.058$\pm$0.000 &0.146$\pm$0.000 \\
ETLMSC    &\underline{0.934$\pm$0.000} &\underline{0.847$\pm$0.000} &\underline{0.934$\pm$0.000}    &\underline{0.709$\pm$0.000}  &\underline{0.774$\pm$0.000} &\underline{0.887$\pm$0.000} \\
LMVSC     &0.892$\pm$0.000 &0.726$\pm$0.000 &0.892$\pm$0.000    &0.355$\pm$0.000 &0.331$\pm$0.000 &0.399$\pm$0.000 \\
FMCNOF    &0.697$\pm$0.119 &0.490$\pm$0.102 &0.711$\pm$0.096    &0.218$\pm$0.033 &0.166$\pm$0.022 &0.221$\pm$0.029 \\
SFMC      &0.916$\pm$0.000 &0.797$\pm$0.000 &0.916$\pm$0.000    &0.188$\pm$0.000 &0.135$\pm$0.000 &0.202$\pm$0.000 \\
FPMVS-CAG   &0.885$\pm$0.000 &0.715$\pm$0.000 &0.885$\pm$0.000    &0.463$\pm$0.000 &0.486$\pm$0.000 &0.481$\pm$0.000 \\
Ours      &\textbf{0.982$\pm$0.000} &\textbf{0.933$\pm$0.000} &\textbf{0.982$\pm$0.000}  &\textbf{0.800$\pm$0.000} &\textbf{0.835$\pm$0.000} &\textbf{0.808$\pm$0.000} \\
\midrule[2pt]
\end{tabular}
%}
\end{table*}

\begin{table}[h]
\caption{Clustering performance on Reuters ("OM" means out of memory, and "-" means the algorithm ran for more than three hours.)}
\label{result2}
\centering
\begin{tabular}{c|ccc}
\toprule[2pt]
Datasets    &\multicolumn{3}{c}{Reuters} \\
\midrule
Metrices   & ACC & NMI & Purity    \\
\midrule[1pt]
CSMSC    &OM &OM &OM         \\
GMC     &- &- &-            \\
ETLMSC  &OM &OM &OM         \\
LMVSC   &0.589$\pm$0.000 &0.335$\pm$0.000 &\underline{0.615$\pm$0.000}        \\
FMCNOF   &0.343$\pm$0.000 &0.125$\pm$0.000 &0.358$\pm$0.000       \\
SFMC     &\underline{0.602$\pm$0.000} &\underline{0.354$\pm$0.000} &0.604$\pm$0.000      \\
FPMVS-CAG  &0.526$\pm$0.000 &0.323$\pm$0.000 &0.603$\pm$0.000    \\
Ours    &\textbf{0.770$\pm$0.000} &\textbf{0.690$\pm$0.000} &\textbf{0.783$\pm$0.000} \\
\midrule[2pt]
\end{tabular}
\end{table}

%\begin{table}[h]
%\caption{Clustering performance on NoisyMnist ("OM" means out of memory, and "-" means the algorithm ran for more than three hours.)}
%\label{result3}
%\centering
%\begin{tabular}{c|ccc}
%\toprule[2pt]
%Datasets    &\multicolumn{3}{c}{NoisyMnist} \\
%\midrule
%Metrices   & ACC & NMI & Purity    \\
%\midrule[1pt]
%CSMSC    &OM &OM &OM         \\
%GMC     &- &- &-            \\
%ETLMSC  &OM &OM &OM         \\
%LMVSC   &0.388$\pm$0.000 &0.344$\pm$0.000 &0.434$\pm$0.000       \\
%FMCNOF   &0.333$\pm$0.000 &0.237$\pm$0.000 &0.340$\pm$0.000    \\
%SFMC     &0.699$\pm$0.000 &0.681$\pm$0.000 &0.727$\pm$0.000      \\
%FPMVS-CAG  &0.554$\pm$0.000 &0.513$\pm$0.000 &0.567$\pm$0.000   \\
%Ours    &&& \\
%\midrule[2pt]
%\end{tabular}
%\end{table}

\begin{figure}[htbp]
	\centering
	\includegraphics[width=0.8\linewidth]{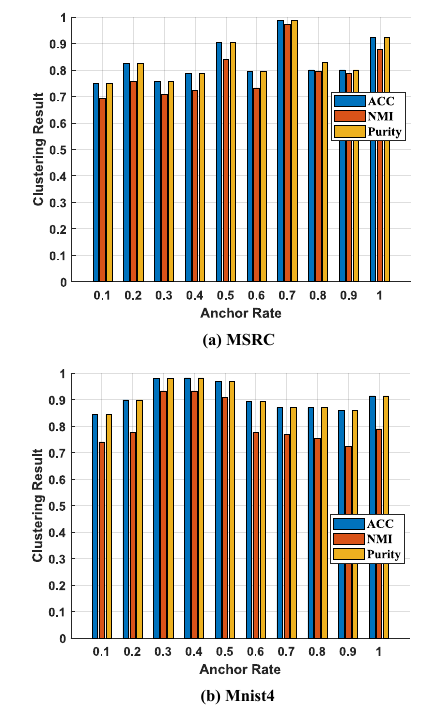}
	\caption{Clustering performance with different anchor rate on MSRC and Mnist4}
	\label{anchorfig}
\end{figure}

\begin{figure}[htbp]
	\centering
	\includegraphics[width=1\linewidth]{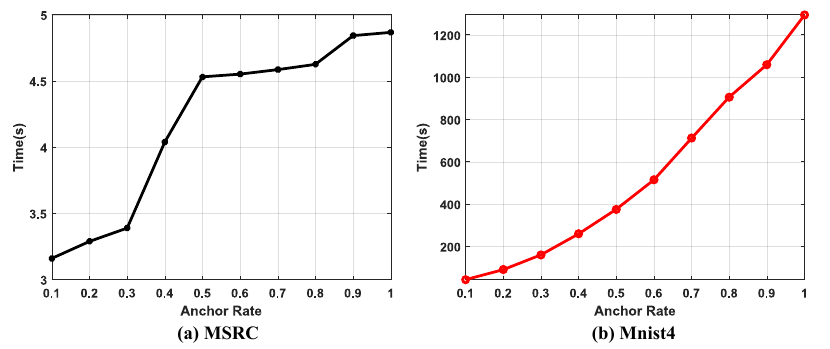}
	\caption{Running time (sec.) with different anchor rate on MSRC and Mnist4}
	\label{timefig}
\end{figure}

\begin{figure}[htbp]
	\centering
	\includegraphics[width=1.0\linewidth]{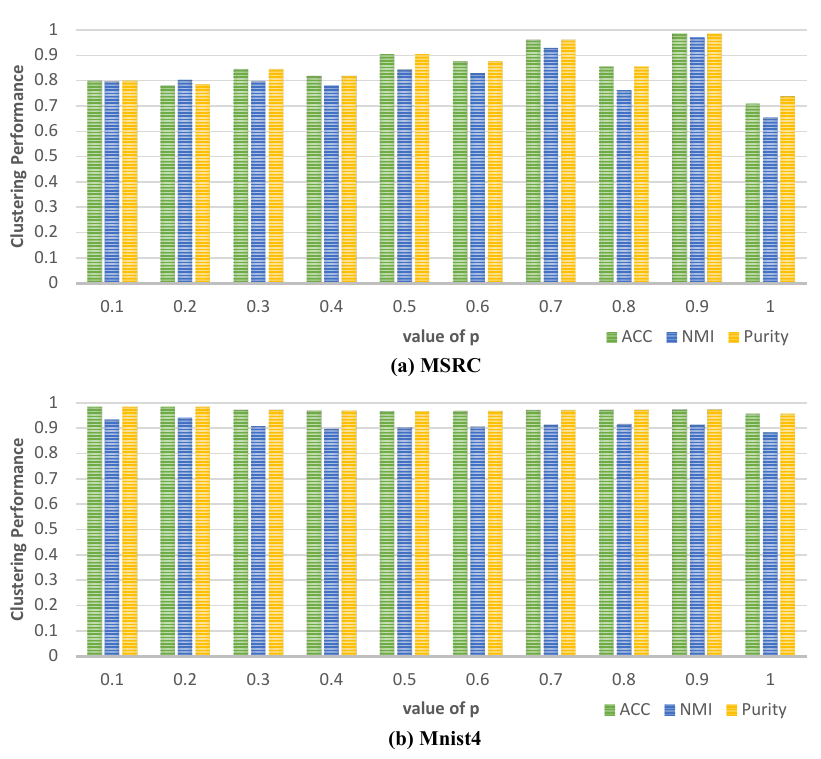}
	\caption{Clustering performance with different $p$ on MSRC and Mnist4}
	\label{pfig}
\end{figure}

\begin{figure}[htbp]
	\centering
	\includegraphics[width=1.0\linewidth]{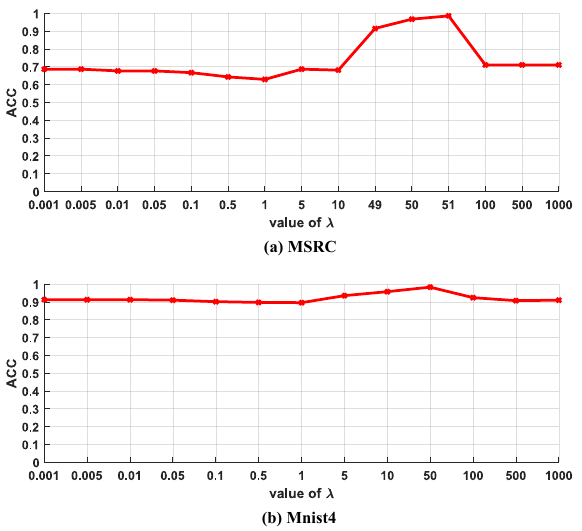}
	\caption{Clustering performance with different $\lambda$ on MSRC and Mnist4}
	\label{lambdafig}
\end{figure}

\subsection{Parameter Analysis}
The hyper-parameters in the model are analyzed experimentally, including the anchor rate (affecting the constructed anchor graph), the value of $p$ in the tensor Schatten $p$-norm, and the value of the coefficient $\lambda$ of the regular term of the tensor Schatten $p$-norm.

The effect of anchor rate on clustering performance is shown in Figure \ref{anchorfig}. We conducted experiments on a small dataset, MSRC, and a medium-sized dataset, Mnist4. They obtained the best indicators at anchor rates of 0.7 and 0.4, respectively. The experimental results show that larger anchor rate is not always better. In addition, according to common sense, the larger the anchor rate, the more time and space it takes to construct the anchor graph.
Figure \ref{timefig} confirms this statement. It can be found that with the increase of the anchor rate, the running time of the algorithm also increases in an approximate linear relationship. It is more obvious on the Mnist4.
It also shows the importance of the choice of anchor rate on the other hand. In general, for large data sets, we tend to use smaller anchor rates.

The influence of the value $p$ on the clustering performance is shown in Figure \ref{pfig}. We start from $p=0.1$ and conduct experiments at intervals of 0.1 until $p=1$. It can be seen from the figure that the best clustering performance is achieved on MSRC and Mnist4 when $p$ is 0.9 and 0.2 respectively. It can also be shown from the experimental point of view that, compared with the nuclear norm (where $p=1$), Schatten $p$-norm minimization can ensure that the rank of the tensor is closer to the target rank. Thus, the complementary information between different views can be better mined, and better clustering performance can be obtained.

Figure \ref{lambdafig} shows the effect of different values of $\lambda$ on clustering performance. Among them, when $\lambda$ achieves 51 and 50, the best clustering results are obtained on MSRC and Mnist4, respectively. Moreover, it can be found from the figure that MSRC is more sensitive to the value of $\lambda$ than Mnist4. When conducting experiments on other small and medium-sized datasets, the value of $\lambda$ is generally adjusted to a small range around 50.

\subsection{Experiments of Convergence}
We tested the convergence of the model, as shown in Figure \ref{convfig}. Since we introduced two auxiliary variables $\bm{\mathcal Q}$ and $\bm{\mathcal J}$ when solving the model, we judged whether the model could gradually converge after enough iterations by calculating the difference between variables $\bm{\mathcal H}$ and $\bm{\mathcal Q}$ and the difference between $\bm{\mathcal H}$ and $\bm{\mathcal J}$ in Eq. (\ref{objective function}). As shown in Figure \ref{convfig}, both differences approach 0 approximately within 50 iterations, so it can be judged that the model proposed in this paper is convergent. At the same time, we also draw the change curve of clustering index ACC with the number of iterations. It can be found from the figure that ACC also converges with the convergence of the model.

\begin{figure}[htbp]
	\centering
	\includegraphics[width=0.75\linewidth]{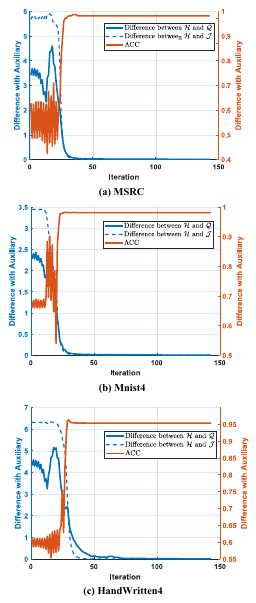}
	\caption{Convergence experiment on MSRC, Mnist4 and HandWritten4}
	\label{convfig}
\end{figure}

\subsection{Visualization of Experimental Results}
We conducted a visualization experiment on the label matrix $\mathbf{H}$ after final fusion according to Algorithm \ref{A1}, as shown in Figure \ref{labelfig}. In the cluster label matrix $\mathbf{H}$ after fusion, the position where the maximum value of each row is located can be regarded as the cluster to which the sample of that row belongs. It can be found that MSRC, Mnist4 and HandWritten4 are clearly divided into 7 clusters, 4 clusters and 10 clusters, respectively.

\begin{figure}[htbp]
	\centering
	\includegraphics[width=0.75\linewidth]{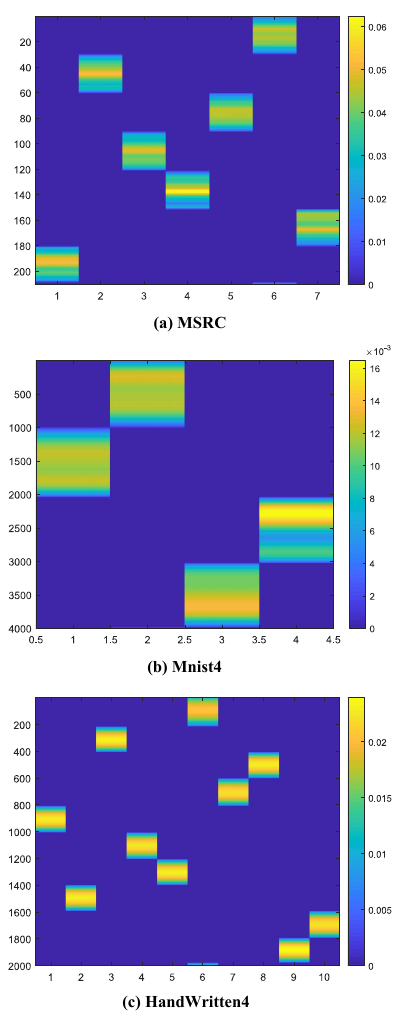}
	\caption{Label visualization on MSRC, Mnist4 and HandWritten4}
	\label{labelfig}
\end{figure}

%------------------------------------------------------------------------
\section{Conclusions}
In this paper, we propose a label learning method based on tensor projection (LLMTP). LLMTP projects the anchor graph space into the label space and thus we can get the clustering results from the label tensor directly. Meanwhile, in order to make full use of the complementary information and spatial structure information between different views, we extend the view-by-view matrix projection process to tensor projection processing multi-view data directly, and use the tensor Schatten $p$-norm to make the clustering label matrix of each view tend to be consistent. Extensive experiments have proved the effectiveness of the proposed method.

\ifCLASSOPTIONcompsoc
\else
\fi
\ifCLASSOPTIONcaptionsoff
  \newpage
\fi

{\small
\bibliographystyle{IEEEtran}
\bibliography{egbib}

\begin{thebibliography}{10}
\providecommand{\url}[1]{#1}
\csname url@samestyle\endcsname
\providecommand{\newblock}{\relax}
\providecommand{\bibinfo}[2]{#2}
\providecommand{\BIBentrySTDinterwordspacing}{\spaceskip=0pt\relax}
\providecommand{\BIBentryALTinterwordstretchfactor}{4}
\providecommand{\BIBentryALTinterwordspacing}{\spaceskip=\fontdimen2\font plus
\BIBentryALTinterwordstretchfactor\fontdimen3\font minus
  \fontdimen4\font\relax}
\providecommand{\BIBforeignlanguage}[2]{{%
\expandafter\ifx\csname l@#1\endcsname\relax
\typeout{** WARNING: IEEEtran.bst: No hyphenation pattern has been}%
\typeout{** loaded for the language `#1'. Using the pattern for}%
\typeout{** the default language instead.}%
\else
\language=\csname l@#1\endcsname
\fi
#2}}
\providecommand{\BIBdecl}{\relax}
\BIBdecl

\bibitem{yun2023low}
\BIBentryALTinterwordspacing
Y.~Yun, J.~Li, Q.~Gao, M.~Yang, and X.~Gao, ``Low-rank discrete multi-view
  spectral clustering,'' \emph{Neural Networks}, vol. 166, pp. 137--147, 2023.
  [Online]. Available: \url{https://doi.org/10.1016/j.neunet.2023.06.038}
\BIBentrySTDinterwordspacing

\bibitem{zhang2023dropping}
Z.~Zhang, Q.~Wang, Z.~Tao, Q.~Gao, and W.~Feng, ``Dropping pathways towards
  deep multi-view graph subspace clustering networks,'' in \emph{Proceedings of
  the 31st {ACM} International Conference on Multimedia}.\hskip 1em plus 0.5em
  minus 0.4em\relax {ACM}, 2023, pp. 3259--3267.

\bibitem{zhao2023contrastive}
\BIBentryALTinterwordspacing
W.~Zhao, Q.~Gao, S.~Mei, and M.~Yang, ``Contrastive self-representation
  learning for data clustering,'' \emph{Neural Networks}, vol. 167, pp.
  648--655, 2023. [Online]. Available:
  \url{https://doi.org/10.1016/j.neunet.2023.08.050}
\BIBentrySTDinterwordspacing

\bibitem{lu2024efficient}
\BIBentryALTinterwordspacing
H.~Lu, H.~Xu, Q.~Wang, Q.~Gao, M.~Yang, and X.~Gao, ``Efficient multi-view
  -means for image clustering,'' \emph{{IEEE} Trans. Image Process.}, vol.~33,
  pp. 273--284, 2024. [Online]. Available:
  \url{https://doi.org/10.1109/TIP.2023.3340609}
\BIBentrySTDinterwordspacing

\bibitem{hao2024ensemble}
\BIBentryALTinterwordspacing
Z.~Hao, Z.~Lu, G.~Li, F.~Nie, R.~Wang, and X.~Li, ``Ensemble clustering with
  attentional representation,'' \emph{{IEEE} Trans. Knowl. Data Eng.}, vol.~36,
  no.~2, pp. 581--593, 2024. [Online]. Available:
  \url{https://doi.org/10.1109/TKDE.2023.3292573}
\BIBentrySTDinterwordspacing

\bibitem{yang2024discrete}
\BIBentryALTinterwordspacing
B.~Yang, J.~Wu, X.~Zhang, X.~Zheng, F.~Nie, and B.~Chen, ``Discrete
  correntropy-based multi-view anchor-graph clustering,'' \emph{Inf. Fusion},
  vol. 103, p. 102097, 2024. [Online]. Available:
  \url{https://doi.org/10.1016/j.inffus.2023.102097}
\BIBentrySTDinterwordspacing

\bibitem{wang2024incomplete}
\BIBentryALTinterwordspacing
Z.~Wang, L.~Li, X.~Ning, W.~Tan, Y.~Liu, and H.~Song, ``Incomplete multi-view
  clustering via structure exploration and missing-view inference,'' \emph{Inf.
  Fusion}, vol. 103, p. 102123, 2024. [Online]. Available:
  \url{https://doi.org/10.1016/j.inffus.2023.102123}
\BIBentrySTDinterwordspacing

\bibitem{tang2023multi}
\BIBentryALTinterwordspacing
C.~Tang, K.~Sun, C.~Tang, X.~Zheng, X.~Liu, J.~Huang, and W.~Zhang,
  ``Multi-view subspace clustering via adaptive graph learning and late fusion
  alignment,'' \emph{Neural Networks}, vol. 165, pp. 333--343, 2023. [Online].
  Available: \url{https://doi.org/10.1016/j.neunet.2023.05.019}
\BIBentrySTDinterwordspacing

\bibitem{zhao2023auto}
\BIBentryALTinterwordspacing
M.~Zhao, W.~Yang, and F.~Nie, ``Auto-weighted orthogonal and nonnegative graph
  reconstruction for multi-view clustering,'' \emph{Inf. Sci.}, vol. 632, pp.
  324--339, 2023. [Online]. Available:
  \url{https://doi.org/10.1016/j.ins.2023.03.016}
\BIBentrySTDinterwordspacing

\bibitem{xiao2023adaptive}
\BIBentryALTinterwordspacing
Q.~Xiao, S.~Du, K.~Zhang, J.~Song, and Y.~Huang, ``Adaptive sparse graph
  learning for multi-view spectral clustering,'' \emph{Appl. Intell.}, vol.~53,
  no.~12, pp. 14\,855--14\,875, 2023. [Online]. Available:
  \url{https://doi.org/10.1007/s10489-022-04267-9}
\BIBentrySTDinterwordspacing

\bibitem{he2023self}
\BIBentryALTinterwordspacing
Y.~He and U.~K. Yusof, ``Self-weighted graph-based framework for multi-view
  clustering,'' \emph{{IEEE} Access}, vol.~11, pp. 30\,197--30\,207, 2023.
  [Online]. Available: \url{https://doi.org/10.1109/ACCESS.2023.3260971}
\BIBentrySTDinterwordspacing

\bibitem{mei2023joint}
\BIBentryALTinterwordspacing
S.~Mei, W.~Zhao, Q.~Gao, M.~Yang, and X.~Gao, ``Joint feature selection and
  optimal bipartite graph learning for subspace clustering,'' \emph{Neural
  Networks}, vol. 164, pp. 408--418, 2023. [Online]. Available:
  \url{https://doi.org/10.1016/j.neunet.2023.04.044}
\BIBentrySTDinterwordspacing

\bibitem{wang2024joint}
\BIBentryALTinterwordspacing
H.~Wang, Q.~Wang, Q.~Miao, and X.~Ma, ``Joint learning of data recovering and
  graph contrastive denoising for incomplete multi-view clustering,''
  \emph{Inf. Fusion}, vol. 104, p. 102155, 2024. [Online]. Available:
  \url{https://doi.org/10.1016/j.inffus.2023.102155}
\BIBentrySTDinterwordspacing

\bibitem{zhan2018graph}
K.~Zhan, C.~Zhang, J.~Guan, and J.~Wang, ``Graph learning for multiview
  clustering,'' \emph{{IEEE} Trans. Cybern.}, vol.~48, no.~10, pp. 2887--2895,
  2018.

\bibitem{li2022multiview}
X.~Li, H.~Zhang, R.~Wang, and F.~Nie, ``Multiview clustering: A scalable and
  parameter-free bipartite graph fusion method,'' \emph{IEEE Transactions on
  Pattern Analysis and Machine Intelligence}, vol.~44, no.~1, pp. 330--344,
  2022.

\bibitem{zhou2022low}
Q.~Zhou, H.~Yang, and Q.~Gao, ``Low-rank constraint bipartite graph learning,''
  \emph{Neurocomputing}, vol. 511, pp. 426--436, 2022.

\bibitem{yang2022multiview}
H.~Yang, Q.~Gao, W.~Xia, M.~Yang, and X.~Gao, ``Multiview spectral clustering
  with bipartite graph,'' \emph{{IEEE} Trans. Image Process.}, vol.~31, pp.
  3591--3605, 2022.

\bibitem{xia2023tensorized}
W.~Xia, Q.~Gao, Q.~Wang, X.~Gao, C.~Ding, and D.~Tao, ``Tensorized bipartite
  graph learning for multi-view clustering,'' \emph{{IEEE} Trans. Pattern Anal.
  Mach. Intell.}, vol.~45, no.~4, pp. 5187--5202, 2023.

\bibitem{gao2020multi}
Q.~Gao, Z.~Wan, Y.~Liang, Q.~Wang, Y.~Liu, and L.~Shao, ``Multi-view projected
  clustering with graph learning,'' \emph{Neural Networks}, vol. 126, pp.
  335--346, 2020.

\bibitem{wang2020robust}
\BIBentryALTinterwordspacing
B.~Wang, Y.~Xiao, Z.~Li, X.~Wang, X.~Chen, and D.~Fang, ``Robust self-weighted
  multi-view projection clustering,'' in \emph{The Thirty-Fourth {AAAI}
  Conference on Artificial Intelligence, 2020}.\hskip 1em plus 0.5em minus
  0.4em\relax {AAAI} Press, 2020, pp. 6110--6117. [Online]. Available:
  \url{https://doi.org/10.1609/aaai.v34i04.6075}
\BIBentrySTDinterwordspacing

\bibitem{sang2022consensus}
X.~Sang, J.~Lu, and H.~Lu, ``Consensus graph learning for auto-weighted
  multi-view projection clustering,'' \emph{Inf. Sci.}, vol. 609, pp. 816--837,
  2022.

\bibitem{wang2022clustering}
\BIBentryALTinterwordspacing
H.~Wang, W.~Zhang, and X.~Ma, ``Clustering of noised and heterogeneous
  multi-view data with graph learning and projection decomposition,''
  \emph{Knowl. Based Syst.}, vol. 255, p. 109736, 2022. [Online]. Available:
  \url{https://doi.org/10.1016/j.knosys.2022.109736}
\BIBentrySTDinterwordspacing

\bibitem{li2023projection}
J.~Li, X.~Zhang, J.~Wang, X.~Wang, Z.~Tan, and H.~Sun, ``Projection-based
  coupled tensor learning for robust multi-view clustering,'' \emph{Inf. Sci.},
  vol. 632, pp. 664--677, 2023.

\bibitem{wang2022align}
S.~Wang, X.~Liu, S.~Liu, J.~Jin, W.~Tu, X.~Zhu, and E.~Zhu, ``Align then
  fusion: Generalized large-scale multi-view clustering with anchor matching
  correspondences,'' in \emph{Advances in Neural Information Processing Systems
  35}, S.~Koyejo, S.~Mohamed, A.~Agarwal, D.~Belgrave, K.~Cho, and A.~Oh, Eds.,
  2022.

\bibitem{kilmer2011factorization}
M.~E. Kilmer and C.~D. Martin, ``Factorization strategies for third-order
  tensors,'' \emph{Linear Algebra and its Applications}, vol. 435, no.~3, pp.
  641--658, 2011.

\bibitem{gao2020enhanced}
Q.~Gao, P.~Zhang, W.~Xia, D.~Xie, X.~Gao, and D.~Tao, ``Enhanced tensor rpca
  and its application,'' \emph{IEEE Transactions on Pattern Analysis and
  Machine Intelligence}, vol.~43, no.~6, pp. 2133--2140, 2021.

\bibitem{zha2020benchmark}
Z.~Zha, X.~Yuan, B.~Wen, J.~Zhou, J.~Zhang, and C.~Zhu, ``A benchmark for
  sparse coding: When group sparsity meets rank minimization,'' \emph{IEEE
  Transactions on Image Processing}, vol.~29, pp. 5094--5109, 2020.

\bibitem{xie2016weighted}
Y.~Xie, S.~Gu, Y.~Liu, W.~Zuo, W.~Zhang, and L.~Zhang, ``Weighted schatten
  $p$-norm minimization for image denoising and background subtraction,''
  \emph{{IEEE} Trans. Image Process.}, vol.~25, no.~10, pp. 4842--4857, 2016.

\bibitem{li2023orthogonal}
J.~Li, Q.~Gao, Q.~WANG, M.~Yang, and W.~Xia, ``Orthogonal non-negative tensor
  factorization based multi-view clustering,'' in \emph{Thirty-seventh
  Conference on Neural Information Processing Systems}, 2023.

\bibitem{Xu2020low}
H.~Xu, X.~Zhang, W.~Xia, Q.~Gao, and X.~Gao, ``Low-rank tensor constrained
  co-regularized multi-view spectral clustering,'' \emph{Neural Networks}, vol.
  132, pp. 245--252, 2020.

\bibitem{yang2021fast}
B.~Yang, X.~Zhang, F.~Nie, F.~Wang, W.~Yu, and R.~Wang, ``Fast multi-view
  clustering via nonnegative and orthogonal factorization,'' \emph{IEEE
  Transactions on Image Processing}, vol.~30, pp. 2575--2586, 2021.

\bibitem{luo2018consistent}
S.~Luo, C.~Zhang, W.~Zhang, and X.~Cao, ``Consistent and specific multi-view
  subspace clustering,'' in \emph{Thirty-second AAAI conference on artificial
  intelligence}, 2018.

\bibitem{wang2019gmc}
H.~Wang, Y.~Yang, and B.~Liu, ``Gmc: Graph-based multi-view clustering,''
  \emph{IEEE Transactions on Knowledge and Data Engineering}, vol.~32, no.~6,
  pp. 1116--1129, 2019.

\bibitem{WuLZ19}
J.~Wu, Z.~Lin, and H.~Zha, ``Essential tensor learning for multi-view spectral
  clustering,'' \emph{IEEE Transactions on Image Processing}, vol.~28, no.~12,
  pp. 5910--5922, 2019.

\bibitem{kang2020large}
Z.~Kang, W.~Zhou, Z.~Zhao, J.~Shao, M.~Han, and Z.~Xu, ``Large-scale multi-view
  subspace clustering in linear time,'' in \emph{Proceedings of the AAAI
  conference on Artificial Intelligence}, 2020, pp. 4412--4419.

\bibitem{wang2021fast}
S.~Wang, X.~Liu, X.~Zhu, P.~Zhang, Y.~Zhang, F.~Gao, and E.~Zhu, ``Fast
  parameter-free multi-view subspace clustering with consensus anchor
  guidance,'' \emph{IEEE Transactions on Image Processing}, vol.~31, pp.
  556--568, 2021.

\end{thebibliography}
}

\end{document}